\documentclass[letterpaper]{article}

\usepackage{natbib,alifeconf}  
\usepackage{graphicx}
\usepackage{amsmath}
\usepackage{amsfonts}
\usepackage{comment}
\usepackage{tikz}
\usetikzlibrary{arrows, decorations.markings}
\let\svtikzpicture\tikzpicture
\def\tikzpicture{\noindent\svtikzpicture}
\usepackage{graphicx}
\usetikzlibrary{calc}
\usepackage{array}
\usepackage{hyperref}
\usepackage{supertabular}
\usepackage{nameref}
\usepackage{amsthm}
\usepackage{amssymb}
\usepackage{varwidth}

\theoremstyle{plain}
\newtheorem{thm}{Theorem}
\newtheorem{defn}[thm]{Definition}
\newtheorem{lemma}[thm]{Lemma}
\newtheorem{ex}[thm]{Example}
\newtheorem{hyp}[thm]{Hypothesis}
\usepackage[linesnumbered,ruled]{algorithm2e}

\def\N{\mathbb{N}}
\def\Z{\mathbb{Z}}

\def\2{{\mathbf{2}}}

\usepackage{colortbl}
\usepackage{multirow}

\makeatletter 
\def\mod@estimate@lineht{%
  \ST@lineht=\arraystretch \baslineskp 
  \ST@stretchht\ST@lineht\advance\ST@stretchht-\baslineskp 
  \ifdim\ST@stretchht<\z@\ST@stretchht\z@\fi 
  \ST@trace\tw@{Average line height: \the\ST@lineht}%
  \ST@trace\tw@{Stretched line height: \the\ST@stretchht}%
} 
 
\makeatother

\newcommand\restr[2]{{
  \left.\kern-\nulldelimiterspace 
  #1 
  \vphantom{\big|} 
  \right|_{#2} 
  }}

\definecolor{myazure}{RGB}{76, 153, 0}

\title{Computational Hierarchy of Elementary Cellular Automata}
\author{Barbora Hudcová$^{1, 2}$ \and Tomáš Mikolov$^{2}$\\
\mbox{}\\
$^1$Charles University, Prague\\
$^2$Czech Institute of Informatics, Robotics and Cybernetics, CTU, Prague} 

\tikzstyle{vecArrow} = [thick, decoration={markings,mark=at position
   1 with {\arrow[semithick]{open triangle 60}}},
   double distance=1.4pt, shorten >= 5.5pt,
   preaction = {decorate},
   postaction = {draw,line width=1.4pt, white,shorten >= 4.5pt}]
\tikzstyle{innerWhite} = [semithick, white,line width=1.4pt, shorten >= 4.5pt]

\begin{document}
\maketitle

\begin{abstract}
The complexity of cellular automata is traditionally measured by their computational capacity. However, it is difficult to choose a challenging set of computational tasks suitable for the parallel nature of such systems. We study the ability of automata to emulate one another, and we use this notion to define such a set of naturally emerging tasks. We present the results for elementary cellular automata, although the core ideas can be extended to other computational systems. We compute a graph showing which elementary cellular automata can be emulated by which and show that certain chaotic automata are the only ones that cannot emulate any automata non-trivially. Finally, we use the emulation notion to suggest a novel definition of chaos that we believe is suitable for discrete computational systems. We believe our work can help design parallel computational systems that are Turing-complete and also computationally efficient.
\end{abstract}

\section{Introduction}
Discrete systems exhibit a wide variety of dynamical behavior ranging from ordered and easily predictable to a very complex, disordered one. In this paper, we study cellular automata (CA) that have intriguing visualizations of their space-time dynamics. Observing them helps us build intuition for distinguishing the different types of dynamics, as studied by \cite{newscience}. Informally, complex CA produce higher-order structures, whereas the space-time diagrams of chaotic CA are seemingly random. Though CA dynamics has been studied extensively (\cite{wolfram}, \cite{kurka}, \cite{wuensche_global}, \cite{gutovitz_hier}, \cite{zenil_compression}), it is still a very difficult problem to formally define the notions of complexity and chaos in CA.

Traditionally, complexity of CA is studied through their computational capacity (\cite{toffoli}, \cite{cook}), whereas chaos is studied via topological dynamics (\cite{devaney}). As the two approaches are not connected in any obvious way, it is an interesting open problem of whether chaotic CA can compute non-trivial tasks (\cite{chaos_comp}). 

In this paper, we study the complexity of CA through their computational capacity; namely, we study their ability to emulate one another; this notion was first introduced by \cite{mazoyer}. We present the \textit{emulation relation} between a pair of CA, and we demonstrate the results on a toy class of elementary CA; namely, we present their computational hierarchy. The results inspired us to define a new notion of chaos for discrete dynamical systems connected to their computational capacity.

\section{Introducing Cellular Automata}

A cellular automaton can be perceived as a $k$-dimensional grid consisting of identical finite state automata with the same local neighborhood. They are all updated synchronously in discrete time steps according to a fixed local update rule which is a function of the neighbors' states. A formal definition can be found in \cite{Kari_survey}.

\subsection{Basic Notions}
We say that $\Z$ is a \textit{one-dimensional cellular grid} and we call its elements the \textit{cells}. Let $S$ be a finite set. An $S$-configuration of the grid is a mapping $c: \Z \rightarrow S$, we write $c_i = c(i)$ for each $i$. We define the \textit{ nearest-neighbors relative neighborhood} of each cell $i \in \Z$ to be the triple $(i-1, i, i+1)$. In this paper, we will study 1-dimensional CA with nearest neighbors. Each such CA is characterized by a tuple $(S, f)$ where $S$ is a finite set of \textit{states} and $f: S^3 \rightarrow S$ is a \textit{local transition rule} of the CA. \textit{The global rule of the CA} $(S, f)$ \textit{operating on an infinite grid} is a mapping $F: S^\Z \rightarrow S^\Z$ defined as:
$$F(c)_i = f(c_{i-1}, c_i, c_{i+1}).$$

For practical purposes, when observing the CA simulations, we consider the grid to be of finite size with a periodic boundary condition. In such a case, we compute the cells in the relative neighborhood modulo the size of the grid.

\paragraph{Elementary CA}
\textit{Elementary cellular automata} (ECA) are 1D nearest-neighbors CA with states $S= \{0, 1 \}$. We identify each local rule $f$ determining an ECA with the \textit{Wolfram number} of $f$ defined as:
$$2^0 f(0, 0, 0) + 2^1 f(0, 0, 1) + 2^2 f(0, 1, 0) + \ldots + 2^7 f(1, 1, 1).$$
We will refer to each ECA as a ``rule $k$'' where $k$ is the corresponding Wolfram number of its underlying local rule. The class of ECA is a frequently used toy model for studying different CA properties due to its relatively small size; there are only 256 of them. 

For an ECA operating on a cyclic grid of size $n$ with global rule $F$ we define the \textit{trajectory of a configuration} $u \in \{0, 1 \}^n$ to be $(u, F(u), F^2(u), \ldots)$.
The space-time diagram of such a simulation is obtained by plotting the configurations as horizontal rows of black and white squares (corresponding to states 1 and 0) with time progressing downwards.

\subsection{CA Complexity via Computational Capacity}
Classically, the complexity of a CA is demonstrated by its computational capacity. Intuitively, we believe that CA capable of computing non-trivial tasks should be more complex than those that are not. In the past, many different computational problems were considered, such as the majority computation task (\cite{mitchellGA}) or the firing squad synchronization (\cite{firing_squad}), a detailed overview was written by \cite{mitchell_overview}. Nevertheless, the most classical task is the simulation of a computationally universal system (a Turing machine, a tag system, etc.). Over the years, many different CA were designed or showed to be Turing complete. However, it seems unnatural to demonstrate the complexity of an inherently parallel system by embedding a sequential computational model into it. In our opinion, an ideal set of benchmark tasks helping us determine the computational capacity of CA should
\begin{itemize}
\item consist of tasks suitable for the parallel computational environment
\item be challenging enough
\item for a given CA and a task $T$, it should be effectively verifiable whether the CA can compute $T$.
\end{itemize}
For a fixed class $\mathcal{C}$ of CA, a natural task is the following:
$$\textit{Given a CA in }  \mathcal{C}, \textit{ how many other CA in } \mathcal{C} \textit{ can it simulate?}$$
The key problem is finding a suitable definition of CA "simulating" one-another. Various approaches have been suggested --- simulations can be interpreted via CA coarse-grainings (\cite{israeli}), or through embedding a local rule to larger size cell-blocks (\cite{mazoyer}, \cite{ollinger_sim}). Many interesting theoretical results stem from such definitions. For instance, \cite{ollinger_sim} defined a CA simulation notion which admits a universal CA --- one that is able to simulate all other CA with the same dimensionality; it was designed by \cite{culik}. In this paper, we consider a notion very similar to the one suggested by \cite{mazoyer}. Compared to their definition, ours is stricter and, thus, slightly easier to verify. We call it \textit{CA emulation} and introduce it in the subsequent section.

\section{CA Emulations}
\subsection{Basic Notions}

\begin{defn}[Subautomaton]~\\
Let $\mathrm{ca}_1 = (S, f)$, $\mathrm{ca}_2 = (T, g)$ be two nearest-neighbor 1D CA. We say that $\mathrm{ca}_1$ is a subautomaton of $\mathrm{ca}_2$ if there exists a one-to-one encoding $e: S \rightarrow T$ such that
$$e(f(s_1, s_2, s_3)) = g(e(s_1), e(s_2), e(s_3))$$
for all $(s_1, s_2, s_3) \in S^3$. 
\end{defn}
This simply means that we can embed the rule table of $\mathrm{ca}_1$ into the rule table of $\mathrm{ca}_2$; the definition is equivalent to saying that $(S, f)$ is a subalgebra of $(T, g)$.

\begin{ex}
For an ECA $\mathrm{ca} = (\{0, 1 \}, f)$, we define the dual ECA $\mathrm{ca}' = (\{0, 1 \}, f')$ as $$f'(b_1, b_2, b_3) = 1 - f(1 - b_1, 1- b_2, 1- b_3).$$ $f'$ simply changes the role of 0 and 1 states. The encoding $e(b) = 1-b$, $b \in \{0, 1 \}$ witnesses that $\mathrm{ca}$ is a subautomaton of $\mathrm{ca}'$ and vice versa. 
\begin{figure}[h!]
\centering
    \begin{tikzpicture}[thick, every node/.style={inner sep=0,outer sep=0}]
  \node at (0, 0) {
     \includegraphics[width=0.3\linewidth]{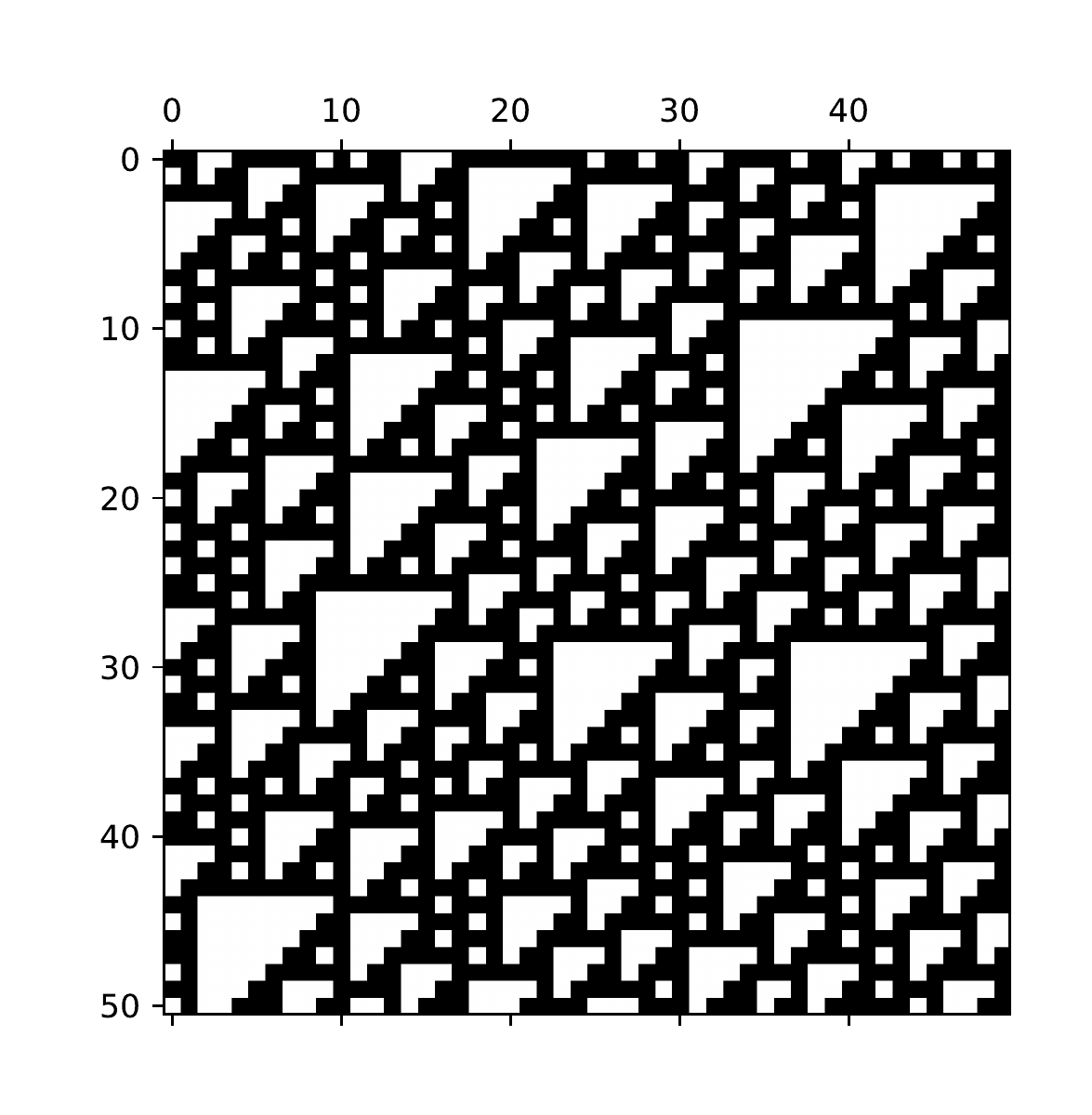}
  };
  \node at (2.5, 0) {
     \includegraphics[width=0.3\linewidth]{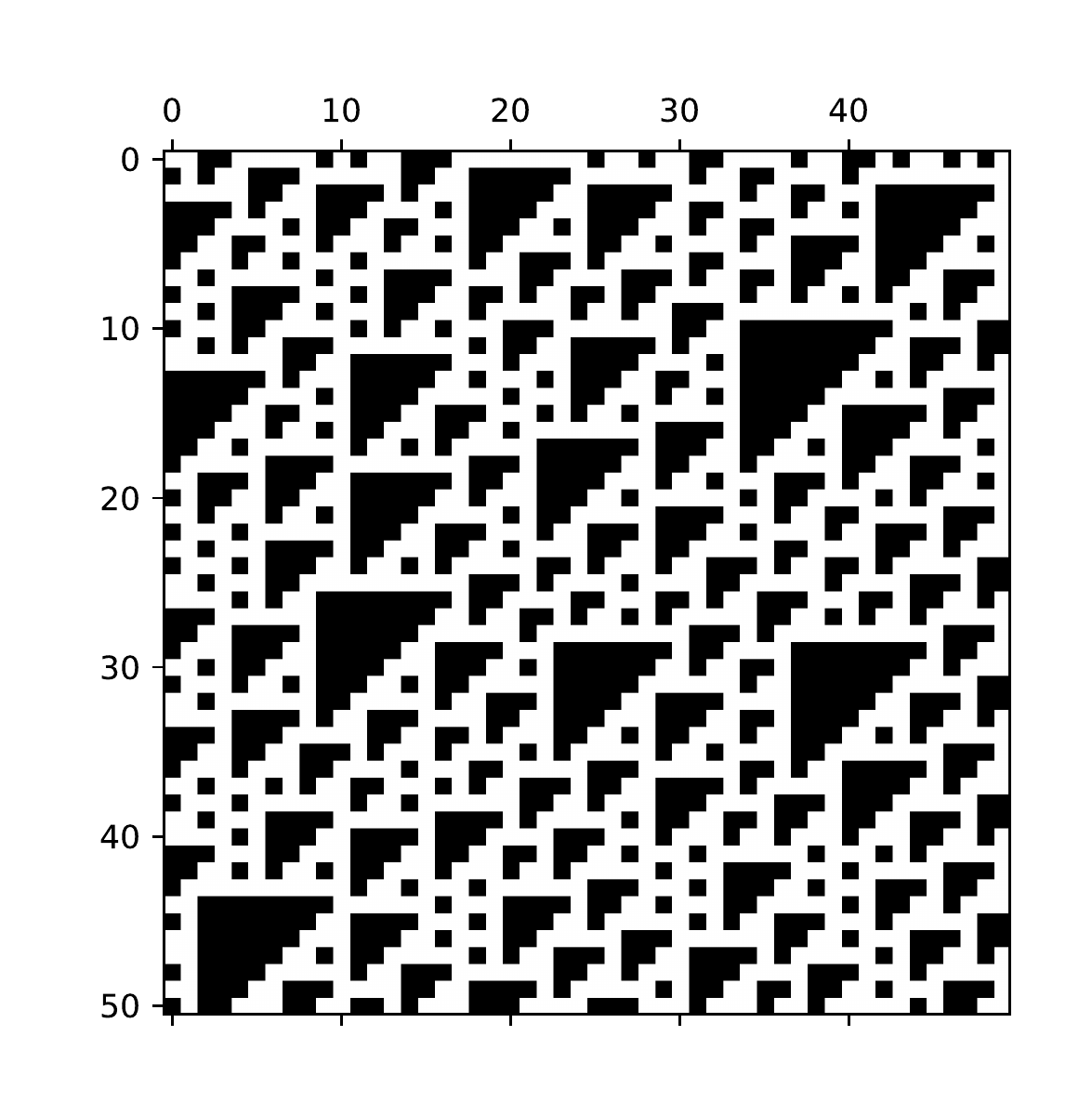}
  };
\end{tikzpicture}
\caption{Space-time diagram of ECA rule 110 on the left and its dual rule 137 on the right.}
\end{figure}
\label{dual}
\end{ex}
Each ECA contains at most two subautomata from the ECA class: itself and its dual ECA. Hence, the subautomaton relation would not produce a rich hierarchy. It is, however, a key concept for defining the CA emulation. Below, we restrict the terminology to ECA, but the theory can be generalized to any 1D CA in a straightforward way.

We will write $\2 = \{0, 1 \}$ and define $\2^+ = \bigcup_{k=1}^\infty \2^k$ to be the set of all finite nonempty binary sequences. 

\begin{defn}[Unravelling a Boolean function]~\\
Let $f: \2^3 \rightarrow \2$. We define $\widetilde{f}: \2^+ \rightarrow \2^+$ as
$$\widetilde{f}(b_1, b_2, \ldots, b_n) = f(b_1, b_2, b_3) \ldots f(b_{n-2}, b_{n-1}, b_n)$$
for any binary sequence $b_1 b_2 \ldots b_n$, $n \geq 3$.
\end{defn}
By $\widetilde{f}^k$ we simply mean the composition
$$\widetilde{f}^k  = \underbrace{\widetilde{f} \circ \widetilde{f} \circ \ldots \circ \widetilde{f}}_{k-\text{times}}.$$ Each iteration of $\widetilde{f}$ shortens the input size by 2, therefore we can notice that when we restrict the domain of $\widetilde{f}^k$ we get
$$\restr{\widetilde{f}^k}{\2^{3k}} : \2^k \times \2^k \times \2^k \rightarrow \2^k.$$
Hence, $\widetilde{f}^k$ can be interpreted as a ternary function operating on \textit{supercells} of $k$ bits and we obtain a CA $(\2^k, \widetilde{f}^k)$ whose dynamics is completely governed by the simple local rule $f$. Therefore, each ECA $(\2, f)$ gives rise to a series of CA

$$ (\2, f),\,(\2^2, \widetilde{f}^2),\,(\2^3, \widetilde{f}^3), \, (\2^4, \widetilde{f}^4), \, \ldots$$

\begin{ex} Suppose we have $f: \2^3 \rightarrow \2$, $f(b_1, b_2, b_3) = b_1 \oplus b_2 \oplus b_3$, where $\oplus$ is the XOR operation. In Figure \ref{ftilde_ex} we show the diagram of $\widetilde{f}^2$ computation.
\begin{figure}[h!]
\centering
\begin{tikzpicture}

\draw[step=0.4cm,color=black] (1.59,0) grid (4,.4);
\fill [fill=black!90] (2,0.01) -- (2.39,0.01) -- (2.39,.39) -- (1.99,.39) -- cycle;
\fill [fill=black!90] (3.6,0.01) -- (3.99,0.01) -- (3.99,.39) -- (3.6,.39) -- cycle;
\draw[orange, very thick] (1.63,.02) rectangle (2.37,.38);
\draw[orange, very thick] (2.43,.02) rectangle (3.17,.38);
\draw[orange, very thick] (3.23,.02) rectangle (3.97,.38);

\draw[step=0.4cm,color=black] (1.99,-.4) grid (3.6,0);
\fill [fill=black!90] (2.01,-.39) -- (2.39,-.39) -- (2.39,-.01) -- (2.01,-.01) -- cycle;
\fill [fill=black!90] (2.41,-.39) -- (2.79,-.39) -- (2.79,-.01) -- (2.41,-.01) -- cycle;
\fill [fill=black!90] (3.21,-.39) -- (3.59,-.39) -- (3.59,-.01) -- (3.21,-.01) -- cycle;

\draw[step=0.4cm,color=black] (2.39,-.8) grid (3.2,-.4);
\draw[orange, very thick] (2.43,-.77) rectangle (3.17,-.43);

\draw[->,>=stealth',semithick, orange] (1.4,.2) arc[radius=.45, start angle=110, end angle=260];
\node at (.9, -.2) {\scriptsize $\widetilde{f}^2$};

\end{tikzpicture}
\caption{ Diagram of $\widetilde{f}^2$ computing on supercells of size 2.}
\label{ftilde_ex}
\end{figure}
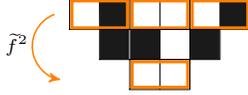
\end{ex}

Let $e: S \rightarrow T$ be a mapping between some finite sets. We define $\overline{e}: S^+ \rightarrow T^+$ simply as
$$\overline{e}(s_1, s_2, \ldots, s_n) = e(s_1), e(s_2), \ldots, e(s_n)$$ for each $s_1, \ldots, s_n \in S$, $n \in \N$.

\begin{defn}[ECA Emulation]
We say that $\mathrm{ca}_1= (\2, f)$ can be emulated by $\mathrm{ca}_2 =  (\2, g)$ with a supercell size $k$, if $\mathrm{ca}_1$ is a subautomaton of $(\2^k, \widetilde{g}^k)$. In such a case, we write $\mathrm{ca}_1 \leq_k \mathrm{ca}_2$.
\end{defn}

Hence, $ (\2, f) \leq_k (\2, g)$ holds if and only if there exists a one-to-one encoding 
$$\mathrm{enc}: \2 \rightarrow \2^k$$ 
such that for any initial configuration $c \in \2^3$ it holds that 
$$\mathrm{enc}(f(c)) = \widetilde{g}^k(\overline{\mathrm{enc}}(c)).$$
We call such encoding the \textit{witnessing encoding}. In other words, the following diagram commutes.
\begin{equation}
\begin{tikzpicture}[scale=0.8]
\node (23) at (-2, 1) {$\2^{3}$};
\node (2k3) at (1, 1) {$(\2^{k})^3$};

\node (2) at (-2, -1) {$\2$};
\node (2k) at (1, -1) {$\2^k$};
\node (circ) at (-.5,0) {$\circ$};

\draw [->] (23) -- (2k3) node [pos=0.5,anchor=north, yshift=13] {$\overline{\mathrm{enc}}$}; 
\draw [->]  (2) -- (2k)  node [pos=0.5,anchor=north, yshift=12] {$\mathrm{enc}$};
\draw [->] (23) -- (2) node [pos=0.5, anchor=west] {$f$};
\draw [->] (2k3) -- (2k) node [pos=0.5,anchor=east, xshift=18] {$\widetilde{g}^k$};  
\end{tikzpicture}
\label{comm_diag}
\end{equation}

We say that $\mathrm{ca}_1$ can be emulated by $\mathrm{ca}_2$ if there exists some $k$ for which $\mathrm{ca}_1 \leq_{k} \mathrm{ca}_2$, and we write  $\mathrm{ca}_1 \leq \mathrm{ca}_2$. If $\mathrm{ca}_1 \leq_1 \mathrm{ca}_2$ we say that $\mathrm{ca}_2$ emulates $\mathrm{ca}_1$ \textit{trivially}. If $\mathrm{ca} \leq_k \mathrm{ca}$ for $k>1$ we say that the $\mathrm{ca}$ is \textit{self-similar}.

In the next section, we give some simple proofs of the key properties of the $\leq$ relation. Namely, we prove that $\leq$ is a preorder and that whenever $\mathrm{ca}_1 \leq \mathrm{ca}_2$ and $\mathrm{ca}_1$ is Turing complete, then $\mathrm{ca}_2$ is also Turing complete.

\subsection{Properties of the Emulation Relation}
\paragraph{Preservation of Turing Completeness}
\begin{lemma}
Suppose that $(\2, f) \leq_k (\2, g)$ with a witnessing encoding $\mathrm{enc}$. Then, for each configuration $c \in \2^l$, $l \geq 3$, it holds that $\overline{\mathrm{enc}}(\widetilde{f}(c)) = \widetilde{g}^k(\overline{\mathrm{enc}}(c))$. 
\end{lemma}

\begin{proof}
Let $c \in \2^l$ for some $l \geq 3$. Then, $\widetilde{f}(c) \in \2^{l-2}$, and $\overline{\mathrm{enc}}(\widetilde{f}(c))$ is a sequence of length $l-2$, each of its elements being a binary $k$-tuple. Below, we show that for each $1 \leq i \leq l-2$ it holds that $\overline{\mathrm{enc}}(\widetilde{f}(c))_i = \widetilde{g}^k(\overline{\mathrm{enc}}(c))_i$.

\begin{align*}
\overline{\mathrm{enc}}(\widetilde{f}(c))_i &= \mathrm{enc}(\widetilde{f}(c)_i)\\
													&= \mathrm{enc}(f(c_{i-1}, c_i, c_{i+1}))\\
 													&= \widetilde{g}^k(\mathrm{enc}(c_{i-1}), \mathrm{enc}(c_i), \mathrm{enc}(c_{i+1}))\\
 													&= \widetilde{g}^k(\overline{\mathrm{enc}}(c))_i.
\end{align*}
\end{proof}
Using the previous result it can easily be shown that for each $t \in \N$ and for each $c \in \2^+$ sufficiently long it holds that:
\begin{equation}
\overline{\mathrm{enc}}(\widetilde{f}^t(c)) = \widetilde{g}^{kt}(\overline{\mathrm{enc}}(c))
\label{sim_kt}
\end{equation}

simply using induction on t.

Therefore, if $(\2, f) \leq_k (\2, g)$ then the diagram (\ref{comm_diag}) commutes for arbitrarily long configurations and for arbitrarily many iterations of the functions $\widetilde{f}$ and $\widetilde{g}^k$. Therefore, any space-time diagram produced by $\mathrm{ca}_1=(\2, f)$ can be efficiently encoded by $\mathrm{enc}$ to a space-time diagram of $\mathrm{ca}_2=(\2, g)$. Hence, if $\mathrm{ca}_1$ is Turing complete, $\mathrm{ca}_2$ must be also. 

\paragraph{Emulation Relation Is a Preorder}
Clearly, for each $\mathrm{ca}$ it holds that $\mathrm{ca} \leq_1 \mathrm{ca}$. Hence, $\leq$ is reflexive. 

\begin{lemma}
The relation $\leq$ is transitive.
\end{lemma}

\begin{proof}
Suppose that $(\2, f) \leq_k (\2, g)$ with witnessing encoding $\varphi$ and $(\2, g) \leq_l (\2, h)$ with witnessing encoding $\psi$ for some $k, \, l \in \N$. We define $\mathrm{enc}= \overline{\psi} \circ \varphi: \2 \rightarrow \2^{kl}$. Then, for any configuration $c \in \2^3$ we have that:
\begin{align*}
\mathrm{enc}(f(c)) &= \overline{\psi} (\varphi(f(c)))\\
			&= \overline{\psi} ( \widetilde{g}^k (\overline{\varphi}(c)))\\
			&= \widetilde{h}^{kl} (\overline{\psi} \circ \overline{\varphi} (c))\\
			&= \widetilde{h}^{kl} (\overline{\mathrm{enc}} (c)).
\end{align*}
In the third equality we use the result from (\ref{sim_kt}). 
\end{proof}
Hence, $\leq$ is transitive and therefore, a preorder. 

In contrast, \cite{mazoyer} define that $(\2, f)$ can be emulated by $(\2, g)$ if there exist $k, \, l \in \N$ such that $(\2^k, \widetilde{f}^k)$ is a subautomaton of $(\2^l, \widetilde{g}^l)$. 
They present much deeper theoretical results about the notion in their paper. In contrast, we concentrate on explicitly computing which ECA can emulate which to form their computational hierarchy and discuss the results.

\paragraph{Emulations as Subalgebras}
For each ECA $(\2, g)$ we have a sequence of algebras:
$$ (\2, g),\,(\2^2, \widetilde{g}^2),\,(\2^3, \widetilde{g}^3), \, (\2^4, \widetilde{g}^4), \, \ldots$$

It holds that $(\2, f) \leq_k (\2, g)$ if and only if $(\2^k, \widetilde{g}^k)$ contains a two-element subalgebra isomorphic to $(\2, f)$; the witnessing encoding being the corresponding algebra homomorphism. 

\cite{israeli} have defined a different notion of CA simulation called the CA coarse-graining. It is interesting to notice that their notion is dual to the CA emulation we have defined in this paper. More specifically, the CA coarse-graining directly corresponds to the congruences of algebras. Namely, $(\2, g)$ can be coarse-grained into $(\2, f)$ if and only if there exists some $k \in \N$ such that $(\2, f)$ is isomorphic to some quotient algebra of $(\2^k, \widetilde{g}^k)$.

The interpretation of CA emulation as subalgebras comes in handy when developing an effective algorithm for computing the $\leq_k$ for a given supercell size $k$.

\subsection{Emulation Computing Algorithm}
Checking whether $\mathrm{ca}_1 \leq_k \mathrm{ca}_2$ gets infeasible for large $k$. In this section, we present an algorithm computing the $\leq_k$ relation for a given $k$ and compare it to the ``naive'' algorithm in terms of efficiency.

\begin{algorithm}[h!]
\footnotesize
    \SetKwInOut{Input}{Input}
    \SetKwInOut{Output}{Output}
    \Input{$\mathrm{ca}_1 = (\2, f), \, \mathrm{ca}_2=(\2, g)$, supercell size $k$}
    \Output{a witnessing $\mathrm{enc}$ if $\mathrm{ca}_1 \leq_k \mathrm{ca}_2$;\\
    "cannot emulate" message otherwise}
    \For{$\mathrm{enc}(0), \, \mathrm{enc}(1) \in \2^k$\\}
      {encoding is valid;\\
      	\For{$i, j, k \in \2$}{
      		\If {$\mathrm{enc}(f(i, j, k)) \neq \widetilde{g}^k(\overline{\mathrm{enc}}(i, j, k))$}
      		{encoding is not valid;\\
      		break
      		}
        }
        \If{encoding is valid}{
      return $\mathrm{enc}$}
      }
      {
        return "cannot emulate"\;
      }
    \caption{Naive algorithm checking whether $\mathrm{ca}_1 \leq_k \mathrm{ca}_2$.}
    \label{alg1}
\end{algorithm}

\begin{algorithm}[h!]
\footnotesize
    \SetKwInOut{Input}{Input}
    \SetKwInOut{Output}{Output}
    \Input{$\mathrm{ca}_2 = (\2, g)$, supercell size $k$}
    \Output{all $\mathrm{ca}_1 = (\2, f)$ together with a witnessing $\mathrm{enc}$
    such that $\mathrm{ca}_1 \leq_k \mathrm{ca}_2$}
    \For{$u, v \in \2^k$\\}
      {$\{u, v \}$ is a valid sublagebra of $(\2^k, \widetilde{g}^k)$;\\
      \For{$i, j, k \in \{u, v \} $}{
      		\If {$\widetilde{g}^k(i, j, k) \not \in \{u ,v \}$}
      		{$\{u ,v \}$ is not a valid subalgebra;\\
      		break}
      		}
      		\If{$\{u, v \}$ is a valid sublagebra}{
      		put $\mathrm{enc}(0)=u, \, \mathrm{enc}(1)=v$;\\
      		determine $f$ for which $f =\mathrm{enc}^{-1} \circ \widetilde{g}^k \circ \overline{\mathrm{enc}}$;\\
      		save ($f, \, \mathrm{enc}$)
      		}
        }

    \caption{Subalgebra algorithm computing all $\mathrm{ca}_1$ for which $\mathrm{ca}_1 \leq_k \mathrm{ca}_2$.}
    \label{alg2}
\end{algorithm}

The naive algorithm simply goes through all possible encodings to determine whether $\mathrm{ca}_1 \leq_k \mathrm{ca}_2$. The subalgebra algorithm computes all two-element subalgebras of $(\2^k, \widetilde{g}^k)$ and thus, determines all $\mathrm{ca}_1$ for which $\mathrm{ca}_1 \leq \mathrm{ca}_2$.

The time complexity with respect to the supercell size $k$ is in $\mathcal{O}(k \cdot 2^{2k})$ for both algorithms. However, Algorithm \ref{alg1} needs to iterate over all ECA to compute the same result as Algorithm \ref{alg2}. Experimentally, we have indeed observed that Algorithm \ref{alg1} takes approximately 250 times longer than Algorithm \ref{alg2} to compute the emulated automata.

\section{Emulation Hierarchy of ECA}
In this section, we present a computational hierarchy based on the emulation relation. We were able to compute $\leq_k$ for supercell size $k$ ranging from 1 to 11 due to the computational limitations. 

In Example \ref{dual} we have seen that for $\mathrm{ca} =(\2, f)$ and its dual $\mathrm{ca}' =(\2, f')$ it holds that $\mathrm{ca} \leq_1 \mathrm{ca}'$ and $\mathrm{ca}' \leq_1 \mathrm{ca}$. Thus, $\leq_1$ is an equivalence relation on the set of ECA, each class containing exactly $f$ and its dual $f'$ (those might coincide for some rules). From the transitivity of $\leq$ we know that $\mathrm{ca}$ can emulate (and be emulated by) exactly the same rules as $\mathrm{ca}'$. Using a simple program, we obtained 136 different ECA equivalence classes given by $\leq_1$. In the following text, we will identify each equivalence class with the ECA it contains having the smaller Wolfram number. We show the hierarchy for such representatives and for supercells of size $k$ ranging from 2 to 11. Specifically, whenever we found that $\mathrm{ca}_1 \leq_k \mathrm{ca}_2$ for some $k \in \{2, 3, \ldots , 11\}$ we represent it by the following diagram. \\~\\
\begin{tikzpicture}
\scriptsize
\centering
\draw (-4, 0) node {};
\draw (0, .8) node (up) [draw, align=left]{\begin{varwidth}{2cm}$\mathrm{ca}_1$\end{varwidth}};
\draw (0, 0) node (down) [draw, align=left]{\begin{varwidth}{2cm}$\mathrm{ca}_2$\end{varwidth}};
\draw  (up) -- (down) node {};
\end{tikzpicture}

Many ECA are capable of emulating simple rules, such as rule 0 or the identity rule 204. Adding so many edges to the diagram would make it unreadable. Hence, we present the hierarchy in three parts.

\newpage
\onecolumn

\begin{figure}[h!]

\begin{tikzpicture}
\scriptsize

\draw (2.2, 12) node (184)[draw, align=left]{\begin{varwidth}{2cm}$184$\end{varwidth}};
\draw (-4, 14) node {\textbf{Main Part of the Hierarchy:}};
\draw  (184) -- (184) node {};

\draw ($ (184) + (0,1) $) node (57)[draw, align=left]{\begin{varwidth}{2cm}$57$\end{varwidth}};
\draw  ($ (57) + (-3,0) $) node (6) [draw, align=left]{\begin{varwidth}{2cm}$6$\end{varwidth}};
\draw ($ (57) + (-2,0) $) node (20)[draw, align=left]{\begin{varwidth}{2cm}$20$\end{varwidth}};
\draw ($ (57) + (-1,0) $) node (56)[draw, align=left]{\begin{varwidth}{2cm}$56$\end{varwidth}};

\draw ($ (57) + (1,0) $) node (98)[draw, align=left]{\begin{varwidth}{2cm}$98$\end{varwidth}};
\draw ($ (57) + (2,0) $) node (148)[draw, align=left]{\begin{varwidth}{2cm}$148$\end{varwidth}};
\draw ($ (57) + (3,0) $) node (134)[draw, align=left]{\begin{varwidth}{2cm}$134$\end{varwidth}};

\draw ($ (134) + (0,1) $) node (97) [draw, align=left]{\begin{varwidth}{2cm}$97$\end{varwidth}};
\draw ($ (148) + (0,1) $) node (41) [draw, align=left]{\begin{varwidth}{2cm}$41$\end{varwidth}};
\draw ($ (148) + (0,-1) $) node (176) [draw, align=left]{\begin{varwidth}{2cm}$176$\end{varwidth}};
\draw ($ (134) + (0,-1) $) node (162) [draw, align=left]{\begin{varwidth}{2cm}$162$\end{varwidth}};

\draw  (134) -- (184) node {};
\draw  (97) -- (134) node {};
\draw  (41) -- (148) node {};

\draw  (148) -- (184) node {};
\draw  (148) -- (176) node {};
\draw  (20) -- (184) node {};
\draw  (98) -- (184) node {};
\draw  (56) -- (184) node {};
\draw  (6) -- (184) node {};
\draw  (57) -- (184) node {};
\draw  (134) -- (162) node {};
\draw [-] (184) edge[loop below]node{} (184);

\draw (-1.0, 8) node (34)[draw, align=left]{\begin{varwidth}{2cm}$34$\end{varwidth}};
\draw ($ (34) + (-1,2.5) $) node (on34) [draw, align=left]{\begin{varwidth}{1.2cm} 2, 10, 66, 42, 46, 172, 130, 138 \end{varwidth}};
\draw ($ (34) + (-1,1.3) $) node (14) [draw, align=left]{\begin{varwidth}{2cm} 14 \end{varwidth}};
\draw ($ (14) + (3.8,0) $) node (188) [draw, align=left]{\begin{varwidth}{2cm} 188 \end{varwidth}};

\draw[opacity=0.2]  (57) -- (34) node {};
\draw[opacity=0.2]  (56) -- (34) node {};
\draw[opacity=0.2]  (6) -- (34) node {};
\draw (on34) -- (34) node {};
\draw (14) -- (34) node {};
\draw (188) -- (34) node {};

\draw ($ (34) + (2.7,0) $) node (48)[draw, align=left]{\begin{varwidth}{2cm}$48$\end{varwidth}};
\draw ($ (on34) + (3.1,0) $) node (on48) [draw, align=left]{\begin{varwidth}{1.2cm} 16, 24, 80, 112,  116,   144, 208, 216 \end{varwidth}};
\draw ($ (14) + (4.6,0) $) node (84) [draw, align=left]{\begin{varwidth}{2cm} 84\end{varwidth}};
\draw ($ (14) + (1.2,0) $) node (onl48) [draw, align=left]{\begin{varwidth}{2cm} 52, 88 \end{varwidth}};
\draw ($ (84) + (.9,0) $) node (152) [draw, align=left]{\begin{varwidth}{2cm} 152 \end{varwidth}};
\draw[opacity=0.2]  (57) -- (48) node {};
\draw[opacity=0.2]  (98) -- (48) node {};
\draw[opacity=0.2]  (20) -- (48) node {};
\draw  (on48) -- (48) node {};
\draw  (152) -- (48) node {};
\draw  (onl48) -- (48) node {};
\draw  (84) -- (48) node {};

\draw ($ (34) + (1.3,0) $) node (15) [draw, align=left]{\begin{varwidth}{2cm}$15$\end{varwidth}};
\draw ($ (on34) + (1.5,0) $) node (on15)[draw, align=left]{\begin{varwidth}{1.2cm}82, 43, 180 53, 142, 11\end{varwidth}};

\draw  (on15) -- (15) node {};
\draw [-] (15) edge[loop below]node{} (15);
\draw[opacity=0.2]  (20) -- (15) node {};
\draw (onl48) -- (15) node {};
\draw  (14) -- (15) node {};
\draw[opacity=0.2]  (148) -- (15) node {};

\draw  ($ (15) + (2.7,0) $) node (85) [draw, align=left]{\begin{varwidth}{2cm}$85$\end{varwidth}};
\draw ($ (on34) + (5,0) $) node (onl85)[draw, align=left]{\begin{varwidth}{1.1cm} 154, 113,  212, 81, 26, 27 \end{varwidth}};

\draw ($ (14) + (2.6,0) $) node (onr85)[draw, align=left]{\begin{varwidth}{1.8cm} 38, 74 \end{varwidth}};

\draw  (onl85) -- (85) node {};
\draw  (onr85) -- (85) node {};
\draw [-] (85) edge[loop below]node{} (85);
\draw[opacity=0.2]  (6) -- (85) node {};
\draw (84) -- (85) node {};
\draw[opacity=0.2]  (134) -- (85) node {};

\draw  (onr85) -- (34) node {};

\draw ($ (48) + (2.5,0) $) node (192)[draw, align=left]{\begin{varwidth}{2cm}$192$\end{varwidth}};
\draw ($ (152) + (2,0) $) node (28) [draw, align=left]{\begin{varwidth}{2cm} 28 \end{varwidth}};
\draw ($ (152) + (1.3,0) $) node (94) [draw, align=left]{\begin{varwidth}{2cm} 94 \end{varwidth}};
\draw ($ (on34) + (6.2,0) $) node (on192) [draw, align=left]{\begin{varwidth}{.8cm} 9, 13, 7, 224 \end{varwidth}};
\draw  (28) -- (192) node {};
\draw (188) -- (192) node {};
\draw (94) -- (192) node {};
\draw[opacity=0.2] (97) -- (192) node {};
\draw (on192) -- (192) node {};
\draw [-] (192) edge[loop below]node{} (192);

\draw ($ (192) + (1.7,0) $) node (136)[draw, align=left]{\begin{varwidth}{2cm}$136$\end{varwidth}};
\draw ($ (28) + (1.8,0) $) node (70) [draw, align=left]{\begin{varwidth}{2cm} 70 \end{varwidth}};
\draw ($ (28) + (.9,0) $) node (156) [draw, align=left]{\begin{varwidth}{2cm} 156 \end{varwidth}};
\draw ($ (on34) + (7.9,0) $) node (on136) [draw, align=left]{\begin{varwidth}{1cm} 21, 65, 69, 168 \end{varwidth}};
\draw  (152) -- (136) node {};
\draw[opacity=0.2]  (41) -- (136) node {};
\draw (70) -- (136) node {};
\draw (94) -- (136) node {};
\draw (156) -- (136) node {};
\draw (on136) -- (136) node {};
\draw (156) -- (192) node {};
\draw [-] (136) edge[loop below]node{} (136);

\draw ($ (136) + (2, 0) $) node (200)[draw, align=left]{\begin{varwidth}{2cm}$200$\end{varwidth}};
\draw ($ (70) + (1,0) $) node (33) [draw, align=left]{\begin{varwidth}{2cm} 33 \end{varwidth}};
\draw ($ (on34) + (9.9,0) $) node (on200) [draw, align=left]{\begin{varwidth}{.8cm} 1, 5, 29, 37 \end{varwidth}};
\draw  (28) -- (200) node {};
\draw  (33) -- (200) node {};
\draw (70) -- (200) node {};
\draw (156) -- (200) node {};
\draw (on200) -- (200) node {};

\draw ($ (200) + (1,0) $) node (132) [draw, align=left]{\begin{varwidth}{2cm} 132 \end{varwidth}};
\draw  (33) -- (132) node {};

\draw  ($ (15) + (1,-1.2) $) node (30) [draw, align=left]{\begin{varwidth}{2cm}$30$\end{varwidth}};

\draw  ($ (30) + (1,0) $) node (45) [draw, align=left]{\begin{varwidth}{2cm}$45$\end{varwidth}};

\draw  ($ (30) + (2,0) $) node (86) [draw, align=left]{\begin{varwidth}{2cm}$86$\end{varwidth}};

\draw  ($ (30) + (3,0) $) node (89) [draw, align=left]{\begin{varwidth}{2cm}$89$\end{varwidth}};

\draw ($ (34) + (-1,0) $) node (102) [draw, align=left]{\begin{varwidth}{2cm}$102$\end{varwidth}};
\draw [-] (102) edge[loop below]node{} (102);

\draw($ (14) + (-1,0) $) node (54) [draw, align=left]{\begin{varwidth}{2cm}$54$\end{varwidth}};
\draw ($ (102) + (-1,0) $) node (50)[draw, align=left]{\begin{varwidth}{2cm}$50$\end{varwidth}};
\draw  (54) -- (50) node {};

\draw ($ (54) + (-1,0) $) node (44) [draw, align=left]{\begin{varwidth}{2cm}$44$\end{varwidth}};
\draw ($ (50) + (-1,0) $) node (12)[draw, align=left]{\begin{varwidth}{2cm}$12$\end{varwidth}};
\draw  (44) -- (12) node {};

\draw  ($ (44) + (-1,0) $)  node (100) [draw, align=left]{\begin{varwidth}{2cm}$100$\end{varwidth}};
\draw ($ (12) + (-1,0) $) node (68)[draw, align=left]{\begin{varwidth}{2cm}$68$\end{varwidth}};
\draw  (100) -- (68) node {};

\draw ($ (132) + (.9,0) $) node (4)[draw, align=left]{\begin{varwidth}{2cm}$4$\end{varwidth}};
\draw ($ (33) + (.7,0) $) node (36)[draw, align=left]{\begin{varwidth}{2cm}$36$\end{varwidth}};
\draw ($ (36) + (.8,0) $) node (104) [draw, align=left]{\begin{varwidth}{2cm}$104$\end{varwidth}};

\draw ($ (104) + (.8,0) $) node (72)[draw, align=left]{\begin{varwidth}{2cm}$72$\end{varwidth}};

\draw  (36) -- (4) node {};
\draw  (104) -- (4) node {};
\draw  (72) -- (4) node {};

\draw ($ (72) + (.8,0) $)  node (108) [draw, align=left]{\begin{varwidth}{2cm}$108$\end{varwidth}};
\draw ($ (4) + (1.6,0) $) node (76)[draw, align=left]{\begin{varwidth}{2cm}$76$\end{varwidth}};
\draw  (108) -- (76) node {};

\end{tikzpicture}
\caption{Emulation Hierarchy of ECA computed for supercell sizes ranging from 2 to 11; main part of the diagram. Some edges are depicted in light gray purely for better understandability. A looped arrow marks self-similar rules.}
\label{main}
\end{figure}
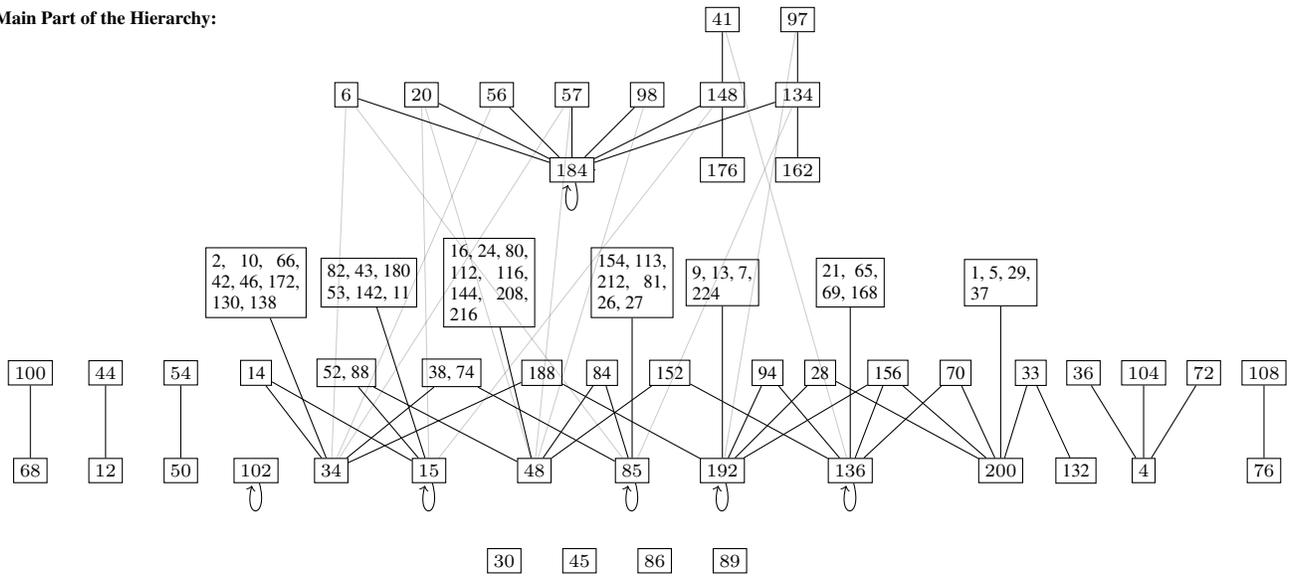

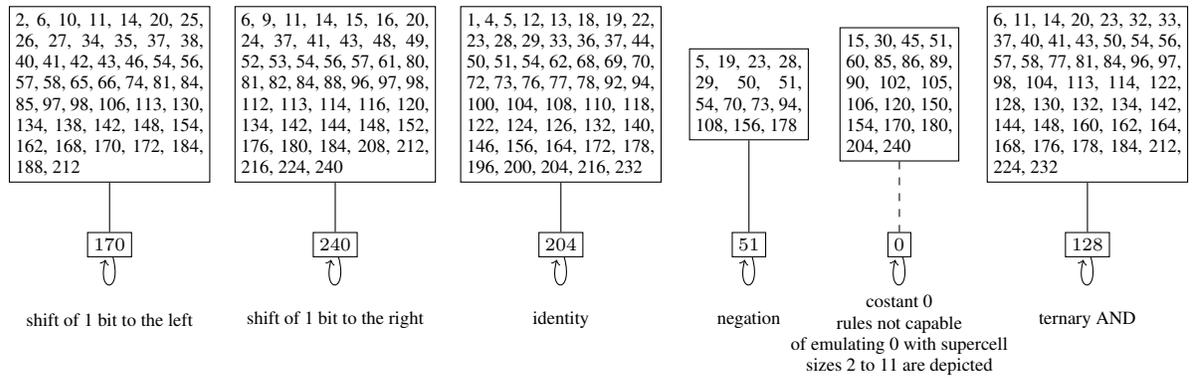
\begin{figure}[h!]

\begin{tikzpicture}
\scriptsize
\draw (-5, 6) node {\textbf{Frequently Emulated Rules:}};

\centering
\draw (-4, 2.5) node (170) [draw, align=left]{\begin{varwidth}{2cm}$170$\end{varwidth}};
\draw ($ (170) + (0,2) $) node (on170) [draw, align=left]{\begin{varwidth}{2.5cm} 2, 6, 10, 11, 14, 20, 25, 26, 27, 34, 35, 37, 38, 40, 41, 42, 43, 46, 54, 56, 57, 58, 65, 66, 74, 81, 84, 85, 97, 98, 106, 113, 130, 134, 138, 142, 148, 154, 162, 168, 170, 172, 184, 188, 212
\end{varwidth}};

\draw ($ (170) + (0,-1) $) node (170text){shift of 1 bit to the left};
\draw  (on170) -- (170) node {};
\draw [-] (170) edge[loop below]node{} (170);

\draw ($ (170) + (3, 0) $) node (240) [draw, align=left]{\begin{varwidth}{2cm}$240$\end{varwidth}};
\draw ($ (240) + (0,2) $) node (on240) [draw, align=left]{\begin{varwidth}{2.5cm} 6, 9, 11, 14, 15, 16, 20, 24, 37, 41, 43, 48, 49, 52, 53, 54, 56, 57, 61, 80, 81, 82, 84, 88, 96, 97, 98, 112, 113, 114, 116, 120, 134, 142, 144, 148, 152, 176, 180, 184, 208, 212, 216, 224, 240 \end{varwidth}};
\draw  (on240) -- (240) node {};
\draw ($ (240) + (0,-1) $) node (240text){shift of 1 bit to the right};
\draw [-] (240) edge[loop below]node{} (240);

\draw ($ (170) + (6, 0) $) node (204) [draw, align=left]{\begin{varwidth}{2cm}$204$\end{varwidth}};
\draw ($ (204) + (0,2) $) node (on204) [draw, align=left]{\begin{varwidth}{2.5cm}1, 4, 5, 12, 13, 18, 19, 22, 23, 28, 29, 33, 36, 37, 44, 50, 51, 54, 62, 68, 69, 70, 72, 73, 76, 77, 78, 92, 94, 100, 104, 108, 110, 118, 122, 124, 126, 132, 140, 146, 156, 164, 172, 178, 196, 200, 204, 216, 232
\end{varwidth}};
\draw ($ (204) + (0,-1) $) node (204text){identity};
\draw  (on204) -- (204) node {};
\draw [-] (204) edge[loop below]node{} (204);

\draw ($ (204) + (2.5,0) $) node (51) [draw, align=left]{\begin{varwidth}{2cm}$51$\end{varwidth}};
\draw ($ (51) + (0,2) $) node (on51) [draw, align=left]{\begin{varwidth}{1.5cm}5, 19, 23, 28, 29, 50, 51, 54, 70, 73, 94, 108, 156, 178
\end{varwidth}};
\draw  (on51) -- (51) node {};
\draw ($ (51) + (0,-1) $) node (51text){negation};
\draw [-] (51) edge[loop below]node{} (51);

\draw ($ (51) + (2,0) $) node (0) [draw, align=left]{\begin{varwidth}{2cm}$0$\end{varwidth}};
\draw ($ (0) + (0,2) $)node (on0) [draw, align=left]{\begin{varwidth}{1.5cm} 15, 30, 45, 51, 60, 85, 86, 89, 90, 102, 105, 106, 120, 150, 154, 170, 180, 204, 240 \end{varwidth}};
\draw ($ (0) + (0,-1.2) $) node (0text){\begin{varwidth}{3cm}  \centering     costant 0\\ rules not capable\\ of emulating 0 with supercell\\ sizes 2 to 11 are depicted \end{varwidth}};
\draw  (on0) -- (0) [dashed] node {};
\draw [-] (0) edge[loop below]node{} (0);

\draw  ($ (0) + (2.5,0) $) node (128) [draw, align=left]{\begin{varwidth}{2cm}$128$ \end{varwidth}};
\draw ($ (128) + (0,2) $) node (on128) [draw, align=left]{\begin{varwidth}{2.5cm} 6, 11, 14, 20, 23, 32, 33, 37, 40, 41, 43, 50, 54, 56, 57, 58, 77, 81, 84, 96, 97, 98, 104, 113, 114, 122, 128, 130, 132, 134, 142, 144, 148, 160, 162, 164, 168, 176, 178, 184, 212, 224, 232
\end{varwidth}};
\draw  (on128) -- (128) node {};
\draw ($ (128) + (0,-1) $) node (128text){ternary AND};
\draw [-] (128) edge[loop below]node{} (128);

\end{tikzpicture}
\caption{Emulation Hierarchy of ECA computed for supercell sizes ranging from 2 to 11, this part shows the most frequently emulated rules.}
\label{trivial}
\end{figure}

\begin{figure}[h!]
\begin{tikzpicture}
\scriptsize

\draw (-5, -6.5) node {\textbf{Emulated Linear Rules:}};

\draw (2, -8.5) node (90) [draw, align=left]{\begin{varwidth}{2cm}$90$\end{varwidth}};

\draw  ($ (90) + (.5,1) $) node (146) [draw, align=left]{\begin{varwidth}{2cm}$146$\end{varwidth}};
\draw($ (146) + (0,1) $) node (22) [draw, align=left]{\begin{varwidth}{2cm}$22$\end{varwidth}};
\draw  ($ (90) + (-1.5, 1) $) node (on90) [draw, align=left]{\begin{varwidth}{2cm}18, 26, 82, 94, 122, 126, 154, 164, 180 \end{varwidth}};

\draw ($ (90) + (2,0) $) node (150) [draw, align=left]{\begin{varwidth}{2cm}$150$\end{varwidth}};
\draw ($ (150) + (0, 1) $) node (105) [draw, align=left]{\begin{varwidth}{2cm}$105$\end{varwidth}};
\draw  ($ (150) + (1, 0) $) node (60) [draw, align=left]{\begin{varwidth}{2cm}$60$\end{varwidth}};
\draw  (105) -- (150) node {};
\draw  (on90) -- (90) node {};
\draw  (22) -- (146) node {};
\draw  (146) -- (90) node {};
\draw [-] (90) edge[loop below]node{} (90);
\draw [-] (150) edge[loop below]node{} (150);
\draw [-] (60) edge[loop below]node{} (60);
\end{tikzpicture}
\caption{Emulation Hierarchy of ECA computed for supercell sizes ranging from 2 to 11, this part shows rules emulating particular linear ECA.}
\label{linear}
\end{figure}
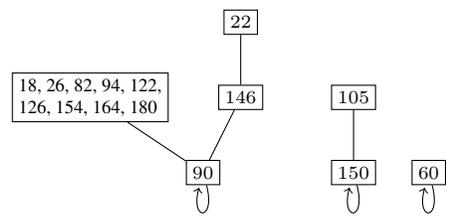
\newpage
\twocolumn 

\paragraph{Main Part}
For compactness, we often show multiple rules in the same node in Figure \ref{main}. However, that does not imply that such rules can emulate one another. Rules in one node are emulated and can emulate exactly the same rules contained in the main part in Figure \ref{main}. However, they do not necessarily emulate the same trivial or linear rules in Figures \ref{trivial} and \ref{linear}. Therefore, we note that rules in the same node do not necessarily have an identical computational capacity.

A task frequently studied in the CA environment is to determine the majority of 0's and 1's in the input configuration. The strict version of the task requires the CA to enter a homogenous state of all 0's or all 1's to indicate the result. It has been shown that no ECA can solve this strict version. If we relax the requirements on the form of the output, \cite{majority_comp} have shown that rule 184 solves the majority task exactly. From the main part of the hierarchy, we can see that by encoding the input configuration in a simple way, at least nine more ECA and their dual rules can solve the task.

\begin{figure}[h!]
\centering
    \begin{tikzpicture}[thick, every node/.style={inner sep=0,outer sep=0}]
  \node at (0, 0) {
     \includegraphics[width=0.315\linewidth]{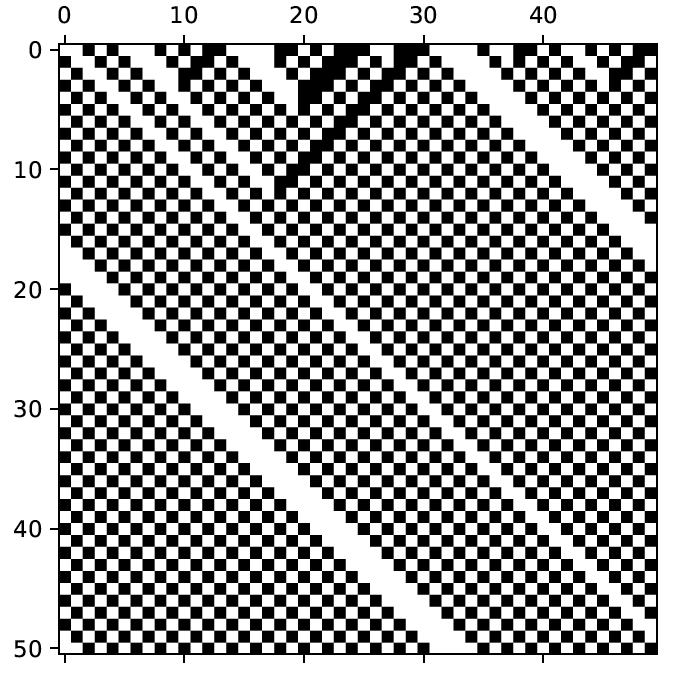}
  };
  \node at (4.2, 0.05) {
     \includegraphics[width=0.7\linewidth]{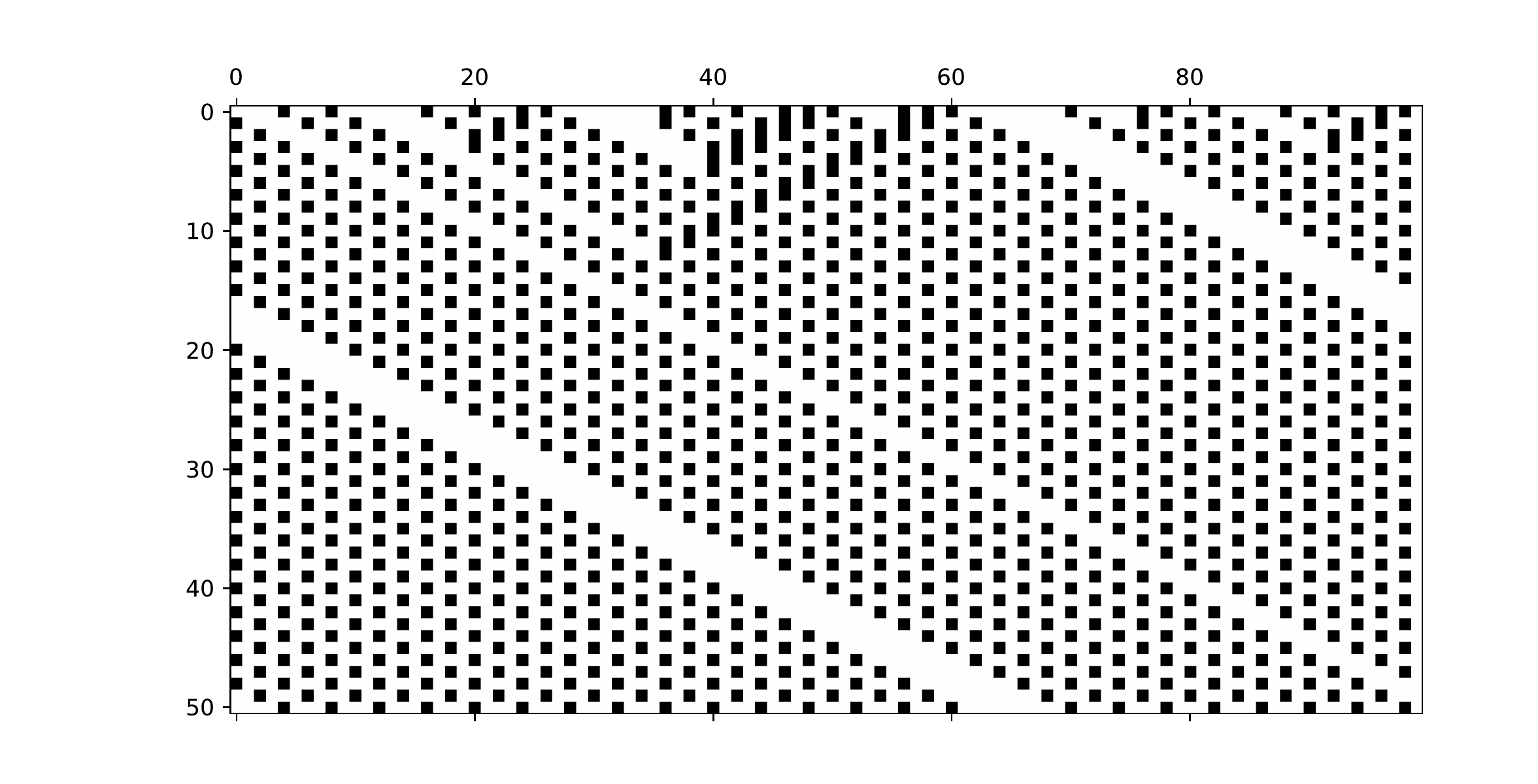}
  };
\end{tikzpicture}
\caption{Space-time diagram of ECA rule 184 on the left. The emulated computation by rule 148 showed on the right. The emulation uses supercells of size 2 and every second time step is depicted.}
\end{figure}

\paragraph{Frequently Emulated Rules}
The CA emulation offers a natural definition of a CA supporting memory: we say that a CA has a \textit{capability of perfect memory} if it can emulate the identity rule 204 (alternatively, we could also tolerate the emulation of shifting rules or the negation). Hence, we have a practical criterion for checking whether a given CA can support memory effectively, i.e., with a small supercell size.
Rules, which cannot emulate either of the rules 51, 204, 170, and 240 with the examined supercell sizes are: 0, 3, 7, 8, 17, 21, 22, 30, 32, 45, 60, 64, 86, 89, 90, 102, 105, 128, 136, 150, 160, 192. We can notice that they include trivial rules (0, 3, 8, 128), linear rules (60, 90, 105, 150), as well as the chaotic rules (22, 30, 45, 86, 89).

\paragraph{Emulating Linear Rules}
We say an ECA is linear, if its local rule $f$ is a linear Boolean function; i. e., it satisfies $f(x+y)= f(x)+f(y)$ for all $x, y \in \2^3$. Such ECA can be studied algebraically which lead to exact results describing e.g., their transients or the structure of attractors (\cite{algCA}, \cite{jenLinearCAAndRecurringSequences}).

\cite{jenExactSolvability} has shown that the nonlinear rules 18, 126, and 146 can be mapped onto the linear rule 90. Figure \ref{linear} agrees with the results and further shows other nonlinear rules with this property which makes them interesting candidates for further algebraic studies.

\paragraph{Bottom Part of the Hierarchy}
$\quad$
\begin{table}[h!]
\centering
\scriptsize
\begin{tabular}{ |m{3.6cm}|m{.5cm}|m{1cm}|  }
 \hline
 \multicolumn{3}{|c|}{\textbf{Top and Bottom of the Emulation Hierarchy}} \\
 \hline
 \centering
 ECA rules &  \multicolumn{2}{l|}{Number of rules they can emulate}
 \tabularnewline
 \hline
41, 97 & 9 & \multirow{3}{*}{ Top}\\
 \arrayrulecolor{black!30}\cline{1-2}
 \arrayrulecolor{black}
6, 20,  54, 57, 134, 148 & 7 &\\
 \arrayrulecolor{black!30}\cline{1-2}
 \arrayrulecolor{black}
14, 37, 56, 84, 94, 98, 156 & 6 &\\
 \hline
0, 3, 8, 17, 60, 64, 90, 102, 105, 106, 120, 150, 170, 204, 240 & 1 & \multirow{2}{*}{Bottom}\\
 \arrayrulecolor{black!30}\cline{1-2}
 \arrayrulecolor{black}
30, 45, 86, 89 & 0 & \\
\hline
\end{tabular}
\caption{ECA Rules at the Top and Bottom of the Emulation Hierarchy computed for supercells of size 1 to 11.}
\label{2D}
\end{table}

There are exactly four rules (and their duals) that seem to be unable to emulate any ECA non-trivially (i.e., with a supercell size larger than one). Rule 86 is obtained from rule 30 by changing the role of the ``left and right neighbor''. Rule 89 is obtained in the same way from the dual of rule 45. All the four rules seem to belong to the most disordered ECA according to metrics such as the compression size of space-time diagrams (\cite{zenil_compression}) or the transient classification (\cite{hudcova}). Due to the seemingly random patterns it produces, rule 30 was implemented as a pseudorandom number generator in Mathematica.

It is interesting that the most chaotic ECA seem to be exactly those unable to emulate any ECA non-trivially. To further explore this, we asked whether the four chaotic CA can emulate \textit{any} 1D nearest-neighbors CA non-trivially, not just the ECA.
 
Let $(S, f)$ be a CA. Since $S$ is finite, there must exist some $s \in S$ and $k \in \N$ such that $\widetilde{f}^k(s, s, s, \ldots, s) = s$. Hence, for any CA with one state $(\{q \}, g)$ the encoding $\mathrm{enc}: q \mapsto (\underbrace{s, s, \ldots, s}_{k-\text{times}})$ witnesses that $(S, f)$ can emulate $(\{q \}, g)$. Therefore, any CA can emulate every CA with one state. The question is whether the four CA can emulate any CA with more than one state non-trivially.

We have examined the sequence of algebras
$$ (\2, f),\,(\2^2, \widetilde{f}^2),\,(\2^3, \widetilde{f}^3), \, (\2^4, \widetilde{f}^4), \, \ldots, (\2^{11}, \widetilde{f}^{11})$$
up to supercell size 11 for rules 30, 45, 86, and 89, and computed all their subalgebras, not only the two-element ones. We found out that all of the four CA only contain subalgebras with one element, hence, they cannot emulate any CA with more than one state for supercell sizes 2 to 11. 

It remains an open problem whether this would hold for supercells of arbitrary size. However, we can conclude that if any of the four chaotic ECA can emulate any non-trivial automaton, it cannot do so effectively, i.e., with a supercell of small size. 

We note that analogous experiments were conducted for the coarse-graining relation. (\cite{dzwinel}) showed that the rules 30 and 45, together with their negations, seem to be exactly those that cannot be coarse-grained non-trivially. These results motivate us to study chaos from a computational perspective in the next section.

\section{Chaos In Cellular Automata}
Intuitively, a chaotic CA has unpredictable dynamics and seemingly random space-time diagrams with no apparent higher-order structures forming. Below, we discuss examples of formal definitions of chaos suggested in the past and propose a new one based on the CA emulation notion.

For discrete systems, chaos is classically studied through topological dynamics, and it is typically connected to the system's sensitivity to initial conditions (\cite{devaney}, \cite{cattaneo}). Other studies of chaos in CA examine e.g., their input entropy (\cite{wuenscheClassifyingCellularAutomataAutomatically}) or properties of the local rule itself (\cite{wuenscheCAEncryption}); a comprehensive study of ECA, including their chaotic behavior, was introduced by \cite{chuaANonlinearDynamicsPerspectiveOfWolframsNKOS}.

From a great review by \cite{martinezANoteOnECAClassification} one can see that some definitions of chaos admit either the shift CA (\cite{devaney}) or linear CA (\cite{cattaneo}) to be chaotic. In the first case, shifting a configuration by one bit does not intuitively feel very unpredictable. For linear automata, it is known that they can be simulated on a computer significantly faster than general CA (\cite{robinson_fast}). Hence, their dynamics can be predicted quite efficiently. Below, we propose a new, much stricter definition of chaotic behavior.

\begin{defn}
We call a 1D nearest-neighbors CA \textit{computationally chaotic} if it cannot emulate any 1D nearest-neighbors CA with more than one state non-trivially.
\end{defn}

Proving that a particular CA is computationally chaotic would, in principle, require verifying an infinite amount of conditions and might require deep theoretical insight into the dynamics of the CA. However, it is useful as a new theoretical notion, which formalizes the unpredictability of a CA. Indeed, if $\mathrm{ca}_1 \leq_k \mathrm{ca}_2$ then for a subset of initial conditions, the result of $\mathrm{ca}_2$ run for $k$ time-steps is equivalent to simply running $\mathrm{ca}_1$ for 1 time-step. Hence, at least some part of its dynamics can be predicted more effectively. In contrast, computationally chaotic CA cannot contain any simpler CA as a substructure in its space-time dynamics. This suggests that no part of its dynamics can be predicted efficiently. Hence, the definition seems to agree with the intuitive understanding of unpredictability.

In CA with trivial dynamics, structures tend to die out quite quickly. This enables such rules to emulate either the rule 0 non-trivially (this can be seen from Figure \ref{trivial}) or to be self-similar (e.g., rules 0, 240, or the shift rule). Thus, they are not computationally chaotic. As linear CA are self-similar, it follows that they are not computationally chaotic either. 

From the results we presented, it follows that the only ECA that could be computationally chaotic are rules 30 and 45, together with their symmetrical rules. Hence, we can form the following hypothesis.

\begin{hyp}
ECA rules 30 and 45 are computationally chaotic.
\end{hyp}

Proving this result would imply that such rules cannot compute any task in the sense of CA emulation.

We note that the definition can be extended to CA in arbitrary dimensions with various neighborhood shapes in a straightforward way. We also note that the definition does not depend on whether the CA operates on a finite or infinite grid and is purely a property of its local rule.

Most definitions of chaos in discrete systems are not related to their computational capacity. Hence, there is an interesting ongoing debate whether chaotic CA can compute any non-trivial tasks. In contrast, the notion of computational chaos is strongly connected to the computational capacity of the CA. If a CA is computationally chaotic, it cannot compute the dynamics of any other CA. Hence, we might conclude that being computationally chaotic directly implies the inability to compute a non-trivial task.

\section{Conclusion}
We studied the notion of CA emulation as a method of determining the computational capacity of CA. We showed that the CA emulation relation forms a preorder on the ECA class and presented an approximation of the emulation hierarchy produced by the preorder. We did notice that the most chaotic CA seem to be the minimal elements of the hierarchy. This inspired us to introduce a new definition of chaos in the CA environment. Our notion of chaos is novel because it is connected to the computational capacity of a system. In contrast to previous concepts of chaos, our definition does not regard linear and shifting automata as chaotic. This agrees with the results that the dynamics of such CA can be predicted more efficiently than for general CA.

The emulation relation can be defined for CA in any dimension and with an arbitrary neighborhood. Though its computation requires verifying infinitely many conditions, we can compute it just for supercells of small size. Verifying such small supercell values helps us determine whether a CA can emulate any others effectively.

A pair of ECA obtained by changing the role of their ``left and right neighbor" has equivalent dynamics; however, the emulation relation we presented is unable to discover such an equivalence. For this reason, it is meaningful to study other possible definitions of a CA subautomaton. It would be very interesting to examine whether the different variants of the definition would change the computational hierarchy results significantly. 

As a possible application, we can use the method of CA emulation to construct Turing-complete CA without having to give elaborate proofs of this fact. We would simply embed a well-known Turing-complete CA into a newly constructed CA with possibly much richer dynamics. It could be interesting to design a CA capable of emulating many different CA and study the dynamics of such a rich system.

\section{Acknowledgements}
Our work was supported by Grant Schemes at CU, reg. no. CZ.02.2.69/0.0/0.0/19\texttt{\char`_}073/0016935, the Ministry of Education, Youth and Sports within the dedicated program ERC CZ under the project POSTMAN with reference LL1902, by the Czech project AI$\&$Reasoning CZ.02.1.01/0.0/0.0/15\texttt{\char`_}003/0000466 and the European Regional Development Fund, by SVV-2020-260589, and is part of the RICAIP project that has received funding from the European Union's Horizon 2020 research and innovation programme under grant agreement No. 857306.

\footnotesize
\bibliographystyle{apalike}
\bibliography{computational_chaos} 

\end{document}